\documentclass[letterpaper, 10 pt, conference]{ieeeconf}  

\IEEEoverridecommandlockouts                              

\usepackage{amssymb} 

\usepackage{amsthm}
\newtheorem{lemma}{Lemma}
\newtheorem*{remark}{Remark}

\usepackage{graphicx}
\usepackage{makecell}
\usepackage{threeparttable} 
\usepackage{subcaption}
\usepackage{xcolor}
\usepackage{booktabs}
\usepackage{multirow}

\usepackage[pagebackref = false,breaklinks,colorlinks,citecolor=green]{hyperref}

\usepackage{xurl}

\let\titleold\title
\renewcommand{\title}[1]{\titleold{#1}\newcommand{\thetitle}{#1}}
\def\maketitlesupplementary
   {
   \newpage
       \twocolumn[
        \centering
        \Large
        \textbf{\thetitle}\\
        \vspace{0.5em}Supplementary Material \\
        \vspace{1.0em}
       ] 
   }

\AtEndPreamble{
    \usepackage[capitalize]{cleveref}
    \crefname{section}{Sec.}{Secs.}
    \Crefname{section}{Section}{Sections}
    \Crefname{table}{Table}{Tables}
    \crefname{table}{Tab.}{Tabs.}
}

\newif\ifproofread
\proofreadfalse

\title{\LARGE \bf
Observability Investigation for Rotational Calibration of (Global-pose aided) VIO under Straight Line Motion
}

\author{Junlin Song, Antoine Richard, and Miguel Olivares-Mendez
\thanks{Space Robotics (SpaceR) Research Group, Int. Centre for Security, Reliability and Trust (SnT), University of Luxembourg, Luxembourg.} 
}

\begin{document}

\maketitle
\thispagestyle{empty}
\pagestyle{empty}

\begin{abstract}

Online extrinsic calibration is crucial for building "power-on-and-go" moving platforms, like robots and AR devices. However, blindly performing online calibration for unobservable parameter may lead to unpredictable results. In the literature, extensive studies have been conducted on the extrinsic calibration between IMU and camera, from theory to practice. It is well-known that the observability of extrinsic parameter can be guaranteed under sufficient motion excitation. Furthermore, the impacts of degenerate motions are also investigated. Despite these successful analyses, we identify an issue with respect to the existing observability conclusion. This paper focuses on the observability investigation for straight line motion, which is a common-seen and fundamental degenerate motion in applications. We analytically prove that pure translational straight line motion can lead to the unobservability of the rotational extrinsic parameter between IMU and camera (at least one degree of freedom). By correcting the existing observability conclusion, our novel theoretical finding disseminates more precise principle to the research community and provides explainable calibration guideline for practitioners. Our analysis is validated by rigorous theory and experiments.

\end{abstract}



\section{Introduction}
\label{sec:intro}

In the last two decades, visual-inertial navigation systems (VINS) have gained great popularity thanks to their ability to provide real-time and precise 6 degree-of-freedom (DoF) motion tracking in unknown GPS-denied or GPS-degraded environments, through the usage of low-cost, low-power, and complementary visual-inertial sensor rigs \cite{mourikis2009vision, delmerico2018benchmark, Meta}. An inertial sensor, IMU, provides high-frequency linear acceleration and local angular velocity measurements of the moving platform, with bias and noise. Therefore, integrating only the IMU measurements to obtain motion prediction inevitably suffers from drift. While visual sensors can estimate IMU bias and reduce the drift of pose estimation by perceiving static visual features from the surrounding environment.

To improve the accuracy, efficiency, robustness or consistency of pose estimation, numerous tightly-coupled visual-inertial odometry (VIO) algorithms have been proposed in the literature. These algorithms can be broadly divided into two categories: optimization-based methods and filter-based methods. Optimization-based methods include OKVIS \cite{leutenegger2015keyframe}, VINS-Mono \cite{qin2018vins}, and ORB-SLAM3 \cite{campos2021orb}. Filter-based methods include ROVIO \cite{bloesch2017iterated}, Multi-State Constraint Kalman Filter (MSCKF) \cite{mourikis2007multi, geneva2020openvins}, and SchurVINS \cite{fan2024schurvins}.

\begin{figure}[htbp]
  \centering
  \includegraphics[width=0.46\textwidth]{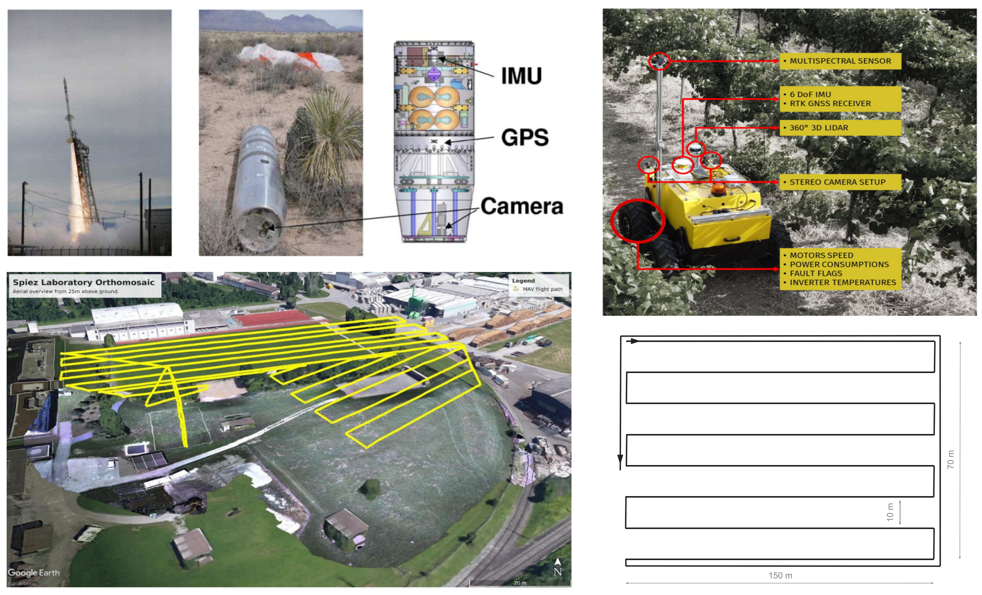}
  \caption{Various straight line movements. Top left: Spacecraft
entry, descent, and landing \cite{mourikis2009vision}. Bottom left: MAV flight path \cite{girod2022state}. Top right:  Agrobot movement in a vineyard field \cite{crocetti2023ard}. Bottom right: Survey followed by Girona 1000 AUV \cite{vial2024lie}.}
  \label{fig: intro}
\end{figure}

Before running, VIO algorithm needs to know the extrinsic parameter between IMU and camera, including 3DoF translational part and 3DoF rotational part, which is a bridge to link measurements from different sensors.
This extrinsic parameter is also critical for other visual perception applications, for example loop correction \cite{qin2018vins}, dense map \cite{zuo2021codevio}, and tracking \cite{qiu2019tracking, eckenhoff2020schmidt}. These visual perception results are represented in camera frame. To transfer these results to the body frame (IMU frame) of the robot or vehicle, accurate extrinsic parameter is desired.
A small misalignment in the extrinsic parameter could generate a large drift and error.

The extrinsic parameter is usually assumed to be rigid and constant, however, this may be not the case in practice. Considering that replacement and maintenance of sensors, and non-rigid deformation caused by mechanical vibration and varying temperature may lead to the alternation of extrinsic parameter, some researchers propose to add extrinsic parameter to state vector to perform \textit{online calibration} \cite{mirzaei2008kalman, kelly2011visual, yang2019degenerate}. If the extrinsic parameter is an observable state variable, online calibration can be resilient with poor prior calibration and converge to true value, which means robustness to the initial value. This feature helps to build "power-on-and-go" moving platforms without the need for repetitive, tedious, manual \textit{offline calibration}.

The success of online extrinsic parameter calibration depends on the observability. Remarkable works have studied the observability of extrinsic parameter between IMU and camera. With the help of artificial visual features on the calibration target board, \cite{mirzaei2008kalman} conclude that extrinsic parameter is observable if the moving platform undergoes at least 2DoF rotational excitation. An interesting corollary from \cite{mirzaei2008kalman} is that the observability of extrinsic parameter is independent of translational excitation. However, the conclusion of \cite{mirzaei2008kalman} is limited by the usage of calibration board, and cannot be applied to real operating environments without calibration board. \cite{kelly2011visual} further extend the calibration of extrinsic parameter with target-less approach, and the conclusion is updated. The moving platform should undergo at least 2DoF motion excitation for both rotation and translation, to ensure the observability of extrinsic parameter.

The above-mentioned observability studies miss the analysis of degenerate motion profiles, which could be occurred and unavoidable in practice. As a supplement, \cite{yang2019degenerate} thoroughly explore the possible degenerate motion primitives and analyze the impact of degenerate motion on the observability of calibration parameters. We note that the rotational extrinsic parameter is summarized as observable for all identified degenerate motions (see Table I in \cite{yang2019degenerate}), except for no motion. However, by observing the top subplot of Fig. 2a in \cite{yang2019degenerate}, we found that the rotational calibration results exhibit unexpected large RMSE (greater than 1 degree) for the case of pure translational straight line motion, which is clearly different from other motion cases. Actually, this distinct curve is an indicator for unobservability.

The inconsistency between the observability conclusion and the calibration results motivates the following research question as the main purpose of this work:

\textbf{\textit{Is the rotational extrinsic parameter of (global-pose aided) VIO observable under pure translational straight line motion?}}

Straight line motions are quite common and fundamental in vehicle driving \cite{jeong2019complex}, agriculture \cite{crocetti2023ard}, coverage survey \cite{girod2022state, vial2024lie}, and planetary exploration \cite{delaune2021range} (see \cref{fig: intro} and \cref{fig: kaist picture}). According to \cite{yang2019degenerate}, if the rotational extrinsic parameter is observable, it is expected that practitioners would straightforward add this parameter to the state vector to perform online calibration, which has been integrated in numerous open-sourced VIO frameworks, like OKVIS \cite{leutenegger2015keyframe}, VINS-Mono \cite{qin2018vins}, ROVIO \cite{bloesch2017iterated}, and Open-VINS \cite{geneva2020openvins}.

However, according to our novel finding (see \cref{table_Correction}), the rotational extrinsic parameter has at least one unobservable DoF when the moving platform undergoes pure translational straight line motion. This implies that performing online rotational calibration is risky, as unobservability can lead to unpredictable and incorrect calibration results. Meanwhile, the misleading observability conclusion in \cite{yang2019degenerate} may have adverse effect on future research. For example, Table III in \cite{yang2023online} is directly inherited from \cite{yang2019degenerate}. Therefore, it is vital to convey more precise principle to the community, otherwise incorrect conclusion would continue to mislead researchers. 
Next, we will verify our observability investigation through rigorous theory and solid experiments.

\begin{table*}
\caption{Observability Investigation for Rotational Calibration of (Global-pose aided) VIO under Straight Line Motion.}
\label{table_Correction}
\begin{center}
\scalebox{1.0}{
\begin{tabular}{|c|c|c|c|c|}
\hline
\multirow{2}{*}{\makecell{Motion}} & \multicolumn{2}{c|}{Pure VIO} & \multicolumn{2}{c|}{Global-pose aided VIO}\\
\cline{2-5}                 & \cite{yang2019degenerate, yang2023online} & Our novel finding & \cite{yang2019degenerate, yang2023online} & Our novel finding\\
\hline
Pure translational straight line motion & observable & at least one unobservable DoF  & observable & at least one unobservable DoF\\
\hline
\makecell{Pure translational straight line motion \\with constant velocity} & observable & fully unobservable & observable & at least one unobservable DoF\\
\hline
\end{tabular}}
\end{center}
\end{table*}

\begin{figure*}
  \centering
  \begin{subfigure}{0.32\linewidth}
    \includegraphics[width=\textwidth, height=0.75\textwidth]{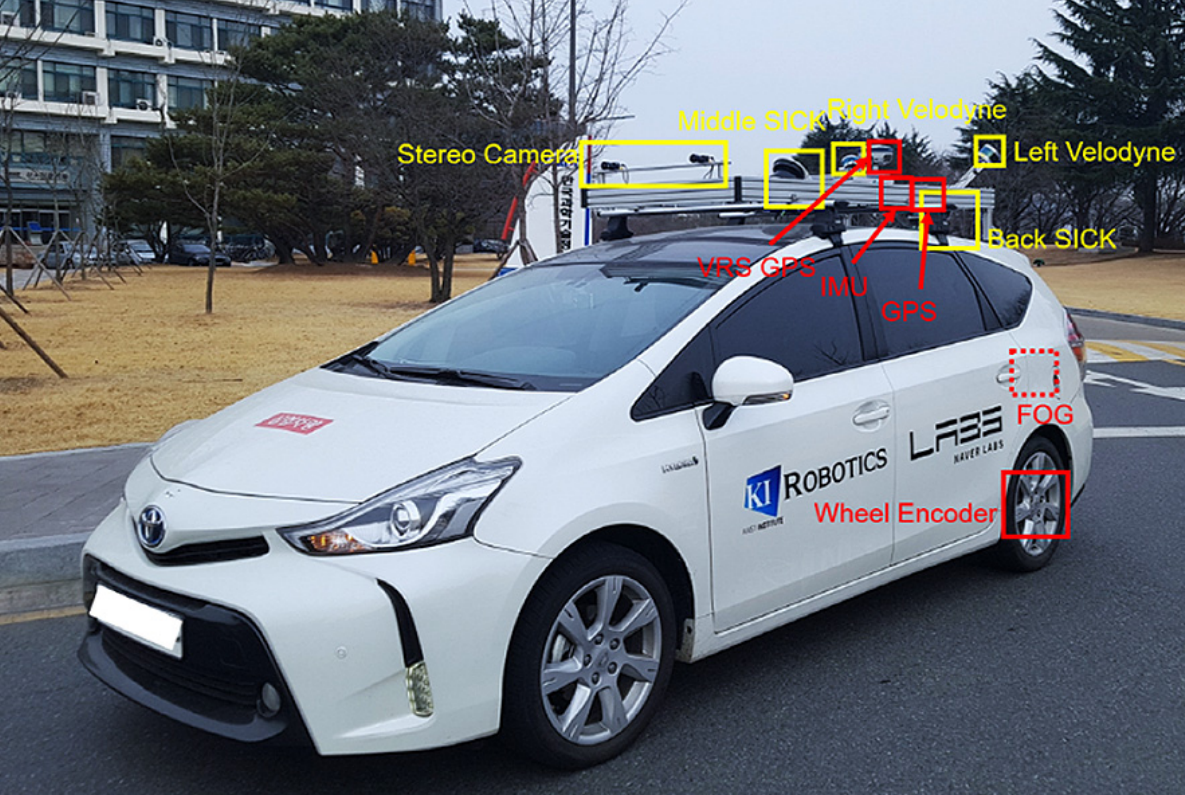}
    \caption{Dataset collection vehicle.}
  \end{subfigure}
  \hfill
  \begin{subfigure}{0.3\linewidth}
    \includegraphics[width=\textwidth, height=0.8\textwidth]{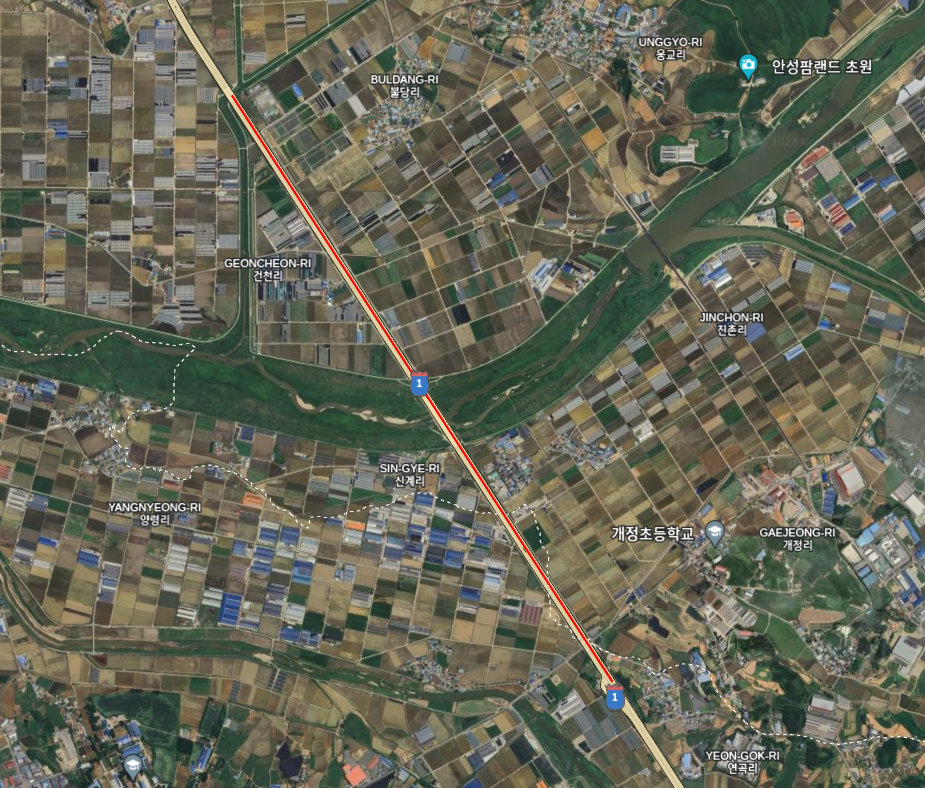}
    \caption{GPS trajectory.}
  \end{subfigure}
  \hfill
  \begin{subfigure}{0.36\linewidth}
    \includegraphics[width=\textwidth, height=0.67\textwidth]{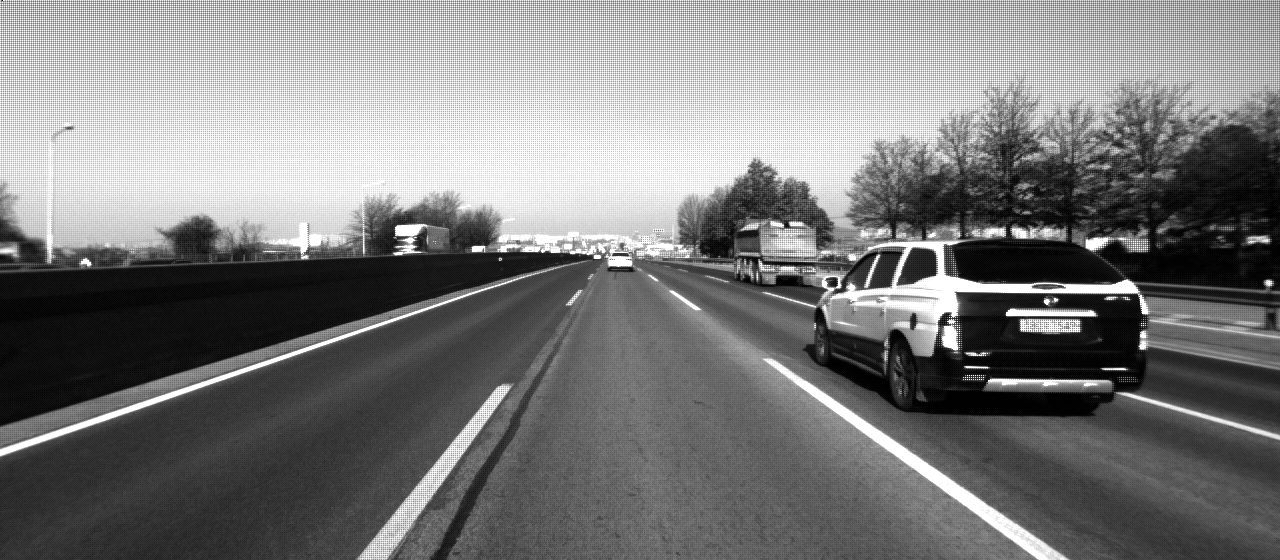}
    \caption{Image from left camera.}
  \end{subfigure}
  
  \caption{Representative pure translational straight line motion from Urban22 sequence in KAIST dataset \cite{jeong2019complex}.}
  \label{fig: kaist picture}
\end{figure*}

\section{Notation} \label{sec:notation}

The main purpose of this paper is to investigate the observability of rotational extrinsic parameter between IMU and camera presented in \cite{yang2019degenerate}. When the moving platform follows a pure translational straight line motion (no rotation), our observability conclusion regarding this rotational extrinsic parameter is different from \cite{yang2019degenerate}. Like \cite{yang2019degenerate}, we consider online calibration of rotational extrinsic parameter (rotational calibration) with two configurations, one is \textbf{pure VIO} and the other is \textbf{global-pose aided VIO}. In the following sections, we will directly analyze the observability matrix in \cite{yang2019degenerate}. As for the construction details of system model, measurement model, and observability matrix, interested readers are advised to refer to \cite{yang2019degenerate, hesch2013consistency}.

The state vector considered in this paper is
\begin{equation} \label{eq:state}
    x = {\left[ {\begin{array}{*{20}{c}}
    {{}_G^I{q^T}}&{b_g^T}&{{}^Gv_I^T}&{b_a^T}&{{}^Gp_I^T}&{{}_I^C{q^T}}&{{}^Gp_f^T}
    \end{array}} \right]^T}
\end{equation}
where ${}_G^I{q}$ represents the orientation of IMU frame $\{ I\} $ with respect to global frame $\{ G\} $, and its corresponding rotation matrix is given by ${}_G^IR$. ${}^Gv_I$ and ${}^Gp_I$ refer to the velocity and position of IMU in frame $\{ G\} $. $b_g$ and $b_a$ represent the gyroscope and accelerometer biases. ${}^Gp_f$ is augmented feature, or SLAM feature \cite{geneva2020openvins}.

${}_I^C{q}$ is rotational calibration parameter, and its corresponding rotation matrix is ${}_I^CR$. Compared to equation (1) in \cite{yang2019degenerate}, $x$ does not include ${}^C{p_I}$ and ${t_d}$, as the online calibration of translational extrinsic parameter, as well as the time offset between IMU and camera, are not the focus of this paper. Our analysis is independent of ${}^C{p_I}$ and ${t_d}$.

In following sections, ${\left[  \bullet  \right]_ \times }$ is denoted as the skew symmetric matrix corresponding to a three-dimensional vector. To simplify the description, the hat symbol $\hat {\left(  \bullet  \right)}$ is omitted, which does not affect observability analysis. Other notations are consistent with \cite{yang2019degenerate}. By assuming that the direction of straight line is denoted as $d$ in the IMU frame $\{ I\} $, we are ready for observability investigation now.

\section{Observability Investigation for Pure VIO}

Referring to equation (21) of \cite{yang2019degenerate}, in the configuration of pure VIO, the observability matrix is
\begin{equation}
    \begin{array}{l}
    {M_k} = {\Xi _k}{\Xi _{{\Gamma _k}}}\\
    {\Xi _{{\Gamma _k}}} = \left[ {\begin{array}{*{20}{c}}
    {{\Gamma _1}}&{{\Gamma _2}}&{ - {I_3}\delta {t_k}}&{{\Gamma _3}}&{ - {I_3}}&{{\Gamma _4}}&{{I_3}}
    \end{array}} \right]
    \end{array}
\end{equation}

Compared to equation (21) of \cite{yang2019degenerate}, the element ${}_{{I_k}}^GR{}_C^IR$ corresponding to ${}^C{p_I}$, and the element ${\Gamma _5}$ corresponding to ${t_d}$, have been removed in ${\Xi _{{\Gamma _k}}}$. The expressions of ${\Gamma _1} \sim {\Gamma _4}$ in ${\Xi _{{\Gamma _k}}}$ are
\begin{equation}
    \begin{array}{l}
    {\Gamma _1} = {\left[ {{}^G{p_f} - {}^G{p_{{I_1}}} - {}^G{v_{{I_1}}}\delta {t_k} + \frac{1}{2}{}^Gg\delta t_k^2} \right]_ \times }{}_{{I_1}}^GR\\
    {\Gamma _2} = {\left[ {{}^G{p_f} - {}^G{p_{{I_k}}}} \right]_ \times }{}_{{I_k}}^GR{\Phi _{I12}} - {\Phi _{I52}}\\
    {\Gamma _3} =  - {\Phi _{I54}}\\
    {\Gamma _4} = {\left[ {{}^G{p_f} - {}^G{p_{{I_k}}}} \right]_ \times }{}_{{I_k}}^GR{}_C^IR
    \end{array}
\end{equation}

The expression of ${\Gamma _1}$ in \cite{yang2019degenerate}, equation (22), has small typos. We have corrected it by referring to equation (53) in \cite{hesch2013consistency}.

In the context of pure translational motion, i.e. no rotation, the orientation of the moving platform does not change at any time. Therefore, ${}_{{I_{\left(  \bullet  \right)}}}^GR$ can be directly represented by ${}_I^GR$ (constant). Referring to equation (114) in \cite{hesch2013consistency}
\begin{equation} \label{eq:I54}
    \begin{array}{l}
    {\Gamma _3} =  - {\Phi _{I54}} = \int_{{t_1}}^{{t_k}} {\int_{{t_1}}^s {{}_{{I_\tau }}^GR} } d\tau ds\\
     = \left( {{}_I^GR} \right)\int_{{t_1}}^{{t_k}} {\int_{{t_1}}^s {\left( 1 \right)} } d\tau ds = \frac{1}{2}{}_I^GR\delta t_k^2
    \end{array}
\end{equation}

The expressions of ${\Gamma _1} \sim {\Gamma _4}$ in ${\Xi _{{\Gamma _k}}}$ become
\begin{equation}
    \begin{array}{l}
    {\Gamma _1} = {\left[ {{}^G{p_f} - {}^G{p_{{I_1}}} - {}^G{v_{{I_1}}}\delta {t_k} + \frac{1}{2}{}^Gg\delta t_k^2} \right]_ \times }{}_I^GR\\
    {\Gamma _2} = {\left[ {{}^G{p_f} - {}^G{p_{{I_k}}}} \right]_ \times }{}_I^GR{\Phi _{I12}} - {\Phi _{I52}}\\
    {\Gamma _3} = \frac{1}{2}{}_I^GR\delta t_k^2\\
    {\Gamma _4} = {\left[ {{}^G{p_f} - {}^G{p_{{I_k}}}} \right]_ \times }{}_I^GR{}_C^IR
    \end{array}
\end{equation}

\begin{lemma} \label{lemma1}
If pure VIO system undergoes pure translational straight line motion, the unobservable directions of ${}_I^CR$ depend on the projection of $d$\footnote{The definition of $d$ is described in the last paragraph of \cref{sec:notation}.} in the camera frame $\{ C\} $. The corresponding right null space of ${M_k}$ is
\begin{equation}
    {N_1} = \left[ {\begin{array}{*{20}{c}}
    {{0_{15 \times 1}}}\\
    {{}_I^CRd}\\
    { - {{\left[ {{}^G{p_f} - {}^G{p_{{I_1}}}} \right]}_ \times }{}_I^GRd}
    \end{array}} \right]
\end{equation}
\end{lemma}

\begin{proof}
Straight line motion indicates the following geometric constraint
\begin{equation}
    {\left[ {{}^{{I_1}}{p_{{I_k}}}} \right]_ \times }d = 0
\end{equation}

Given the above constraint, we first verify that ${N_1}$ belongs to the right null space of ${\Xi _{{\Gamma _k}}}$.
\begin{equation}
    \begin{array}{l}
    {\Xi _{{\Gamma _k}}}{N_1} = {\Gamma _4}{}_I^CRd - {\left[ {{}^G{p_f} - {}^G{p_{{I_1}}}} \right]_ \times }{}_I^GRd\\
     = {\left[ {{}^G{p_f} - {}^G{p_{{I_k}}}} \right]_ \times }{}_I^GRd - {\left[ {{}^G{p_f} - {}^G{p_{{I_1}}}} \right]_ \times }{}_I^GRd\\
     =  - {\left[ {{}^G{p_{{I_k}}} - {}^G{p_{{I_1}}}} \right]_ \times }{}_I^GRd
    \end{array}
\end{equation}

One geometric relationship can be utilized
\begin{equation}
    {}^G{p_{{I_k}}} = {}^G{p_{{I_1}}} + {}_{{I_1}}^GR{}^{{I_1}}{p_{{I_k}}} = {}^G{p_{{I_1}}} + {}_I^GR{}^{{I_1}}{p_{{I_k}}}
\end{equation}

Subsequently
\begin{equation}
    \begin{array}{l}
    {\Xi _{{\Gamma _k}}}{N_1} =  - {\left[ {{}^G{p_{{I_k}}} - {}^G{p_{{I_1}}}} \right]_ \times }{}_I^GRd\\
     =  - {\left[ {{}_I^GR{}^{{I_1}}{p_{{I_k}}}} \right]_ \times }{}_I^GRd\\
     =  - {}_I^GR{\left[ {{}^{{I_1}}{p_{{I_k}}}} \right]_ \times }{}_I^G{R^T}{}_I^GRd\\
     =  - {}_I^GR{\left[ {{}^{{I_1}}{p_{{I_k}}}} \right]_ \times }d
     = 0
    \end{array}
\end{equation}

Finally
\begin{equation}
     \Rightarrow {M_k}{N_1} = {\Xi _k}{\Xi _{{\Gamma _k}}}{N_1} = 0
\end{equation}

Hence, ${N_1}$ belongs to the right null space of ${M_k}$. ${N_1}$ indicates that the unobservable directions of ${}_I^CR$ are dependent on the non-zero components of ${}_I^CRd$.
\end{proof}

\begin{lemma} \label{lemma2}
If pure VIO system undergoes pure translational straight line motion with constant velocity, the 3DoF of ${}_I^CR$ are all unobservable. The corresponding right null space of ${M_k}$ is
\begin{equation}
    {N_2} = \left[ {\begin{array}{*{20}{c}}
    {{}_G^IR}\\
    {{0_3}}\\
    {{0_3}}\\
    { - {}_G^IR{{\left[ {{}^Gg} \right]}_ \times }}\\
    {{0_3}}\\
    { - {}_I^CR{}_G^IR}\\
    {{0_3}}
    \end{array}} \right]
\end{equation}
\end{lemma}

\begin{proof}
Straight line motion with constant velocity indicates the following geometric constraint
\begin{equation}
    {}^G{p_{{I_k}}} = {}^G{p_{{I_1}}} + {}^G{v_{{I_1}}}\delta {t_k}
\end{equation}

Given the above constraint, we first verify that ${N_2}$ belongs to the right null space of ${\Xi _{{\Gamma _k}}}$.
\begin{equation}
    \begin{array}{l}
    {\Xi _{{\Gamma _k}}}{N_2} = {\Gamma _1}{}_G^IR - {\Gamma _3}{}_G^IR{\left[ {{}^Gg} \right]_ \times } - {\Gamma _4}{}_I^CR{}_G^IR\\
     = {\left[ {{}^G{p_f} - {}^G{p_{{I_1}}} - {}^G{v_{{I_1}}}\delta {t_k} + \frac{1}{2}{}^Gg\delta t_k^2} \right]_ \times }\\
     {\quad} - \frac{1}{2}{\left[ {{}^Gg} \right]_ \times }\delta t_k^2 - {\left[ {{}^G{p_f} - {}^G{p_{{I_k}}}} \right]_ \times }\\
     = {\left[ {{}^G{p_{{I_k}}} - {}^G{p_{{I_1}}} - {}^G{v_{{I_1}}}\delta {t_k}} \right]_ \times }
     = 0
    \end{array}
\end{equation}

Finally
\begin{equation}
     \Rightarrow {M_k}{N_2} = {\Xi _k}{\Xi _{{\Gamma _k}}}{N_2} = 0
\end{equation}

Hence, ${N_2}$ belongs to the right null space of ${M_k}$. ${N_2}$ indicates that the 3DoF of ${}_I^CR$ are all unobservable.
\end{proof}

\begin{remark}
We note that the rotational extrinsic parameter ${}_I^CR$ has at least one degree of freedom that is unobservable when the platform undergoes pure translational straight line motion. More specifically, when moving with constant velocity, the 3 degrees of freedom of ${}_I^CR$ are completely unobservable. When moving with variable velocity, at least one degree of freedom is unobservable as $\left\| {{}_I^CRd} \right\| \ne 0$. 
\end{remark}

\section{Observability Investigation for Global-pose aided VIO}

Like \cite{yang2019degenerate}, the observability of rotational extrinsic parameter is also discussed in the configuration of global-pose aided VIO. Our conclusion is different from \cite{yang2019degenerate}.
Referring to equation (40) of \cite{yang2019degenerate}, the observability matrix is
\begin{equation}
    \scalebox{0.95}{$
    \begin{array}{l}
    M_k^{\left( g \right)} = \Xi _k^{\left( g \right)}\Xi _{{\Gamma _k}}^{\left( g \right)}\\
    \Xi _{{\Gamma _k}}^{\left( g \right)} = \left[ {\begin{array}{*{20}{c}}
    {{\Gamma _1}}&{{\Gamma _2}}&{ - {I_3}\delta {t_k}}&{{\Gamma _3}}&{ - {I_3}}&{{\Gamma _4}}&{{I_3}}\\
    {{\Phi _{I11}}}&{{\Phi _{I12}}}&{{0_3}}&{{0_3}}&{{0_3}}&{{0_3}}&{{0_3}}\\
    {{\Phi _{I51}}}&{{\Phi _{I52}}}&{{\Phi _{I53}}}&{{\Phi _{I54}}}&{{I_3}}&{{0_3}}&{{0_3}}
    \end{array}} \right]
    \end{array}
    $}
\end{equation}

The last two rows of $\Xi _{{\Gamma _k}}^{\left( g \right)}$ in \cite{yang2019degenerate} is incorrect. We have corrected it by multiplying the measurement Jacobian matrix with the state transition matrix. Detailed derivations are provided in \cref{sec: additional_correction} of supplementary material \cite{song2025observability}.

\begin{lemma} \label{lemma3}
If global-pose aided VIO system undergoes pure translational straight line motion, the unobservable directions of ${}_I^CR$ depend on the projection of $d$ in the camera frame $\{ C\} $. The corresponding right null space of $M_k^{\left( g \right)}$ is ${N_1}$.
\end{lemma}

\begin{proof}
A naive way of finding the corresponding right null space for $M_k^{\left( g \right)}$ is to test the product of $\Xi _{{\Gamma _k}}^{\left( g \right)}$ and ${N_1}$
\begin{equation}
    \Xi _{{\Gamma _k}}^{\left( g \right)}{N_1} = \left[ {\begin{array}{*{20}{c}}
    {{\Gamma _4}{}_I^CRd - {{\left[ {{}^G{p_f} - {}^G{p_{{I_1}}}} \right]}_ \times }{}_I^GRd}\\
    {{0_{3 \times 1}}}\\
    {{0_{3 \times 1}}}
    \end{array}} \right]
\end{equation}

According to \cref{lemma1}
\begin{equation}
    \Xi _{{\Gamma _k}}^{\left( g \right)}{N_1} = 0
\end{equation}

Finally
\begin{equation}
     \Rightarrow M_k^{\left( g \right)}{N_1} = \Xi _k^{\left( g \right)}\Xi _{{\Gamma _k}}^{\left( g \right)}{N_1} = 0
\end{equation}

Hence, ${N_1}$ belongs to the right null space of $M_k^{\left( g \right)}$. ${N_1}$ indicates that the unobservable directions of ${}_I^CR$ are dependent on the non-zero components of ${}_I^CRd$.
\end{proof}

\begin{lemma} \label{lemma4}
If global-pose aided VIO system undergoes pure translational straight line motion with constant velocity, the unobservable directions of ${}_I^CR$ depend on the projection of $d$ in the camera frame $\{ C\} $. The corresponding right null space of $M_k^{\left( g \right)}$ is still ${N_1}$.
\end{lemma}

\begin{proof}
A naive way of finding the corresponding right null space for $M_k^{\left( g \right)}$ is to test the product of $\Xi _{{\Gamma _k}}^{\left( g \right)}$ and ${N_2}$
\begin{equation}
    \Xi _{{\Gamma _k}}^{\left( g \right)}{N_2} = \left[ {\begin{array}{*{20}{c}}
    {{\Gamma _1}{}_G^IR - {\Gamma _3}{}_G^IR{{\left[ {{}^Gg} \right]}_ \times } - {\Gamma _4}{}_I^CR{}_G^IR}\\
    {{\Phi _{I11}}{}_G^IR}\\
    {{\Phi _{I51}}{}_G^IR - {\Phi _{I54}}{}_G^IR{{\left[ {{}^Gg} \right]}_ \times }}
    \end{array}} \right]
\end{equation}

According to \cref{lemma2}
\begin{equation}
    \Xi _{{\Gamma _k}}^{\left( g \right)}{N_2} = \left[ {\begin{array}{*{20}{c}}
    {{0_3}}\\
    {{\Phi _{I11}}{}_G^IR}\\
    {{\Phi _{I51}}{}_G^IR - {\Phi _{I54}}{}_G^IR{{\left[ {{}^Gg} \right]}_ \times }}
    \end{array}} \right]
\end{equation}

Referring to equation (46) in \cite{hesch2013consistency}, ${\Phi _{I11}} \ne 0$, it is clear that $\Xi _{{\Gamma _k}}^{\left( g \right)}{N_2} \ne 0$. Therefore, the unobservable direction $N_2$ is no longer hold due to the inclusion of global pose measurement.

It is worth noting that \cref{lemma4} is a special case of \cref{lemma3}. Hence, ${N_1}$ still belongs to the right null space of $M_k^{\left( g \right)}$. ${N_1}$ indicates that the unobservable directions of ${}_I^CR$ are dependent on the non-zero components of ${}_I^CRd$.
\end{proof}

\begin{remark}
We note that the rotational extrinsic parameter ${}_I^CR$ has at least one degree of freedom that is unobservable when the platform undergoes pure translational straight line motion, regardless of variable velocity or constant velocity. In the case of constant velocity, the unobservable directions can be decreased with the aides of global pose measurement, compared to the pure VIO configuration. More specifically, in the global-pose aided VIO configuration, the worst case is three degrees of freedom are unobservable, while the best case is only one degree of freedom is unobservable. The difference between our conclusion and \cite{yang2019degenerate} is marked in \cref{table_Correction}.
\end{remark}

\section{Results}

We conduct verification experiments based on Open-VINS \cite{geneva2020openvins}. As this paper focuses on the observability investigation of the rotational extrinsic parameter, we only perform online calibration for the rotational extrinsic parameter and set the translational extrinsic parameter and time offset as true values, referring to our state vector (\cref{eq:state}).

\subsection{Comments on results in \cite{yang2019degenerate}} \label{sec: Comments}

Table I from \cite{yang2019degenerate} show that the rotational extrinsic parameter is observable for pure translational motion. However, we find that \cite{yang2019degenerate} actually did not perform theoretical analysis on the rotational calibration (${}_I^CR$). Besides that, it can be seen from the top subplot of Fig. 2a in \cite{yang2019degenerate}, if the simulation trajectory is a pure translational straight line motion with constant velocity, the calibration result of the rotational extrinsic parameter shows large RMSE (greater than 1 degree). Regarding the inconsistency between observability assertion and simulation result, no ablation experiments were conducted, by calibrating the rotational extrinsic parameter only and turning off the calibration of the translational extrinsic parameter and time offset. Moreover, Section VI of \cite{yang2019degenerate} did not validate the convergence consistency with different initial ${}_I^CR$. Section VII of \cite{yang2019degenerate} missed the verification of pure translational straight line motion in real-world experiments.

\subsection{Numerical Study} \label{sec: Numerical Study}

\begin{table*}
  \caption{Final calibration results of the rotational extrinsic parameter for pure VIO system undergoes pure translational straight line motion with variable velocity. The absolute errors of roll, pitch, and yaw at 60s, are recorded with different perturbations.}
  \centering
  \scalebox{1.0}{
  \begin{tabular}{@{}cccccccccc@{}}
    \toprule
    \multirow{2}{*}{\makecell{Perturbations of\\ (roll, pitch, yaw)}} & \multicolumn{3}{c}{Case-1} & \multicolumn{3}{c}{Case-2} & \multicolumn{3}{c}{Case-3} \\
    \cmidrule(lr){2-4} \cmidrule(lr){5-7} \cmidrule(lr){8-10}
    & {Roll} & {Pitch} & {Yaw} & {Roll} & {Pitch} & {Yaw} & {Roll} & {Pitch} & {Yaw} \\
    \midrule
    (2, -4, -5) & 11.31 & 0.04 & 0.03 & 5.28 & 2.17 & 0.05 & 4.67 & 1.83 & 2.21\\
    (-4, 3, 3) & 0.37 & 0.05 & 0.02 & 2.16 & 0.91 & 0.03 & 4.04 & 1.69 & 1.69\\
    (5, -2, -1) & 13.38 & 0.04 & 0.02 & 7.02 & 2.89 & 0.09 & 1.01 & 0.41 & 0.53\\
    (-1, -5, -3) & 8.96 & 0.03 & 0.02 & 2.67 & 1.10 & 0.01 & 4.99 & 1.94 & 2.23\\
    (3, 0, 1) & 9.55 & 0.04 & 0.01 & 4.60 & 1.89 & 0.06 & 0.40 & 0.19 & 0.20\\
    (1, 2, -4) & 7.34 & 0.06 & 0.04 & 3.52 & 1.44 & 0.05 & 5.98 & 2.44 & 2.63\\
    (0, 5, 2) & 2.51 & 0.05 & 0.02 & 1.60 & 0.64 & 0.07 & 1.22 & 0.57 & 0.52\\
    (-3, 4, 0) & 1.58 & 0.05 & 0.03 & 0.63 & 0.28 & 0.03 & 4.82 & 1.98 & 2.08\\
    (-5, 1, 4) & 0.44 & 0.04 & 0.01 & 3.20 & 1.35 & 0.02 & 6.09 & 2.50 & 2.53\\
    (4, -1, 5) & 9.20 & 0.02 & 0.01 & 4.49 & 1.84 & 0.05 & 1.12 & 0.39 & 0.51\\
    (-2, -3, -2) & 7.36 & 0.04 & 0.02 & 0.77 & 0.30 & 0.01 & 5.89 & 2.34 & 2.56\\
    \hline
    Avg & 6.55 & 0.04 & 0.02 & 3.27 & 1.35 & 0.04 & 3.66 & 1.48 & 1.61\\
    \bottomrule
  \end{tabular}}
  \label{tab: variable velocity-vio}
\end{table*}

\begin{figure*}
  \centering
  \begin{subfigure}{0.96\linewidth}
    \includegraphics[width=\textwidth, height=0.2\textwidth]{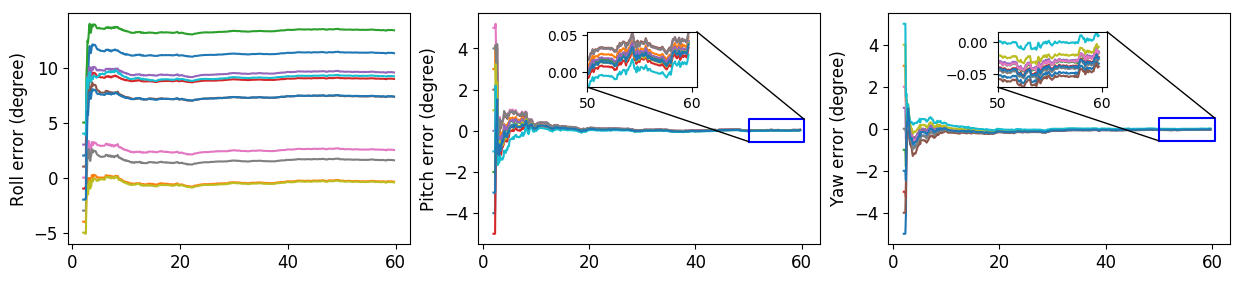}
  \end{subfigure}
  
  \begin{subfigure}{0.96\linewidth}
    \includegraphics[width=\textwidth, height=0.2\textwidth]{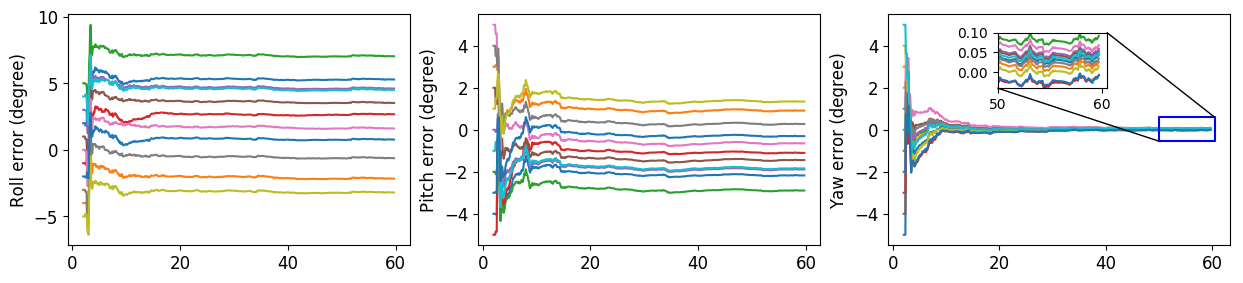}
  \end{subfigure}

  \begin{subfigure}{0.96\linewidth}
    \includegraphics[width=\textwidth, height=0.2\textwidth]{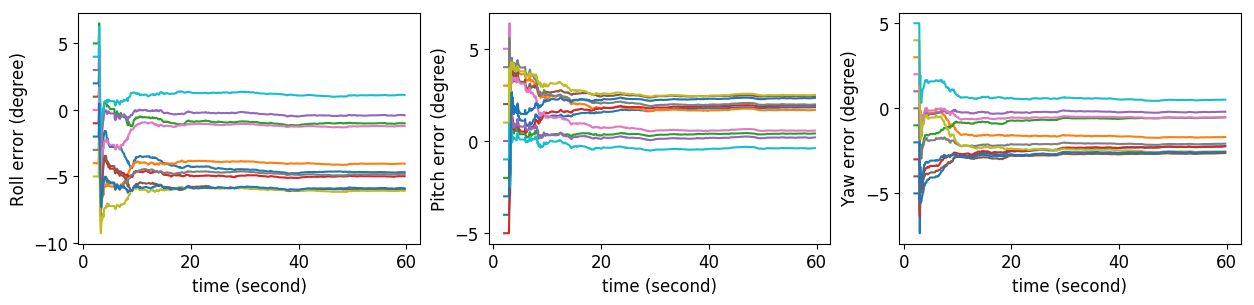}
  \end{subfigure}

  \caption{Calibration results for pure VIO system undergoes pure translational straight line motion with variable velocity. $y$-axis represents errors of the rotational calibration parameter over time respect to different initial guesses. $x$-axis represents time in seconds. Top to bottom corresponds to Case-1 to Case-3 in \cref{sec: Numerical Study}.}
  \label{fig: variable velocity-vio}
  
\end{figure*}

Employing the Open-VINS simulator and importing the desired 6DoF trajectory, realistic multi-sensor data are generated for experiments under two different configurations. For the pure VIO configuration, we generate IMU measurements at 400 Hz and image measurements at 10 Hz. For the global-pose aided VIO configuration, additional 10 Hz global-pose measurements are generated. The global-pose measurement noises are defined as
\begin{equation}
    \begin{array}{l}
    n_{p} \sim {\cal N}\left( {{0_{3 \times 1}},{\sigma_p^2 I_3}} \right), \sigma_p = 0.1m \\
    n_{\theta} \sim {\cal N}\left( {{0_{3 \times 1}},{\sigma_\theta^2 I_3}} \right), \sigma_\theta = 0.1rad
    \end{array}
\end{equation}
where $n_{p}$ and $n_{\theta}$ represent Gaussian noises for global position and orientation measurement, respectively. 

This paper focuses on pure translational straight line motion, therefore the orientation of the input trajectory, ${}_G^IR$, is set as ${I_3}$. To validate the observability assertion summarized in \cref{table_Correction}, two types of straight line motion with different velocity profiles are designed as

\begin{itemize}

\item Trajectory-1: ${}^G{p_I} = {\left[ {\begin{array}{*{20}{c}}
{2\cos \left( {\frac{\pi }{5}t} \right)}&0&0
\end{array}} \right]^T}$.

\item Trajectory-2: ${}^G{p_I} = {\left[ {\begin{array}{*{20}{c}}
{0.5t}&0&0
\end{array}} \right]^T}$.

\end{itemize}

Trajectory-1 corresponds to variable velocity motion, while Trajectory-2 corresponds to constant velocity motion. The direction vector corresponding to both these two trajectories is $d = {\left[ {\begin{array}{*{20}{c}}
1&0&0
\end{array}} \right]^T}$. As the unobservable directions of ${}_I^CR$ may depend on the non-zero components of ${}_I^CRd$, three types of groundtruth ${}_I^CR$ are designed as

\begin{itemize}

\item Case-1: 
\\
\scalebox{0.9}{${}_I^CR = \left[ {\begin{array}{*{20}{c}}
1&0&0\\
0&1&0\\
0&0&1
\end{array}} \right],{}_I^CRd = \left[ {\begin{array}{*{20}{c}}
1\\
0\\
0
\end{array}} \right]$}.

\item Case-2: 
\\
\scalebox{0.9}{${}_I^CR = \left[ {\begin{array}{*{20}{c}}
0.707&0.707&0\\
-0.707&0.707&0\\
0&0&1
\end{array}} \right],{}_I^CRd = \left[ {\begin{array}{*{20}{c}}
0.707\\
-0.707\\
0
\end{array}} \right]$}.

\item Case-3: 
\\
\scalebox{0.9}{${}_I^CR = \left[ {\begin{array}{*{20}{c}}
0.5&0.707&-0.5\\
-0.5&0.707&0.5\\
0.707&0&0.707
\end{array}} \right],{}_I^CRd = \left[ {\begin{array}{*{20}{c}}
0.5\\
-0.5\\
0.707
\end{array}} \right]$}.

\end{itemize}

For each case, we initialize ${}_I^CR$ by adding different perturbations to the three degrees of freedom of ${}_I^CR$ (roll, pitch, and yaw), and collect calibration error with respect to groundtruth ${}_I^CR$. The range of perturbation is $\left[ { - {{5.0}^ \circ },{{5.0}^ \circ }} \right]$. If a certain degree of freedom is observable, it should be robust to different perturbations, namely, the calibration error should consistently converge to 0. On the contrary, if it is unobservable, the calibration error can not converge to 0 and is expected to be sensitive to the initial value.

Firstly, we analyze the calibration results for Case-1 of Trajectory-1 in the pure VIO configuration, as shown in the \cref{tab: variable velocity-vio}. Pitch and yaw exhibit observable characteristic, while roll not. This is because the non-zero component of ${}_I^CRd$ corresponds to roll. For Case-2, yaw exhibits observable characteristic, while roll and pitch not. This is because non-zero components of ${}_I^CRd$ correspond to roll and pitch. For Case-3, roll, pitch, and yaw all exhibit unobservable characteristic. This is because none of the three components of ${}_I^CRd$ are zero. The calibration results over time are shown in the \cref{fig: variable velocity-vio}. Similar analysis also applies to different combinations of  configurations and trajectories, please refer to \cref{table_Numerical} and \cref{sec: additional_numerical_results} of supplementary material \cite{song2025observability} for other results. These results successfully validate that our novel observability conclusions are correct. Overall, observable degree of freedom shows deterministic behavior, i.e. converging to groundtruth over time, while unobservable degree of freedom exhibits unpredictable behavior.

\begin{table*}
\caption{Numerical Study Results for Rotational Calibration of (Global-pose aided) VIO under Straight Line Motion.}
\label{table_Numerical}
\begin{center}
\scalebox{1.0}{
\begin{tabular}{|c|c|c|c|c|}
\hline
\multirow{2}{*}{\makecell{Motion}} & \multicolumn{2}{c|}{Pure VIO} & \multicolumn{2}{c|}{Global-pose aided VIO}\\
\cline{2-5}                 & calibration results & conclusion & calibration results & conclusion\\
\hline
\makecell{Trajectory-1 \\in \cref{sec: Numerical Study}} & \cref{tab: variable velocity-vio} and \cref{fig: variable velocity-vio} & at least one unobservable DoF  & \makecell{\cref{tab: variable velocity-global vio} and \cref{fig: variable velocity-global vio} in \\supplementary material \cite{song2025observability}} & at least one unobservable DoF\\
\hline
\makecell{Trajectory-2 \\in \cref{sec: Numerical Study}} & \makecell{\cref{tab: constant velocity-vio} and \cref{fig: constant velocity-vio} in \\supplementary material \cite{song2025observability}} & fully unobservable & \makecell{\cref{tab: constant velocity-global vio} and \cref{fig: constant velocity-global vio} in \\supplementary material \cite{song2025observability}} & at least one unobservable DoF\\
\hline
\end{tabular}}
\end{center}
\end{table*}

\subsection{Real-world Dataset}

\begin{figure}[htbp]
  \centering
    \begin{subfigure}{0.23\textwidth}
        \centering
        \includegraphics[width=\textwidth, height=0.8\textwidth]{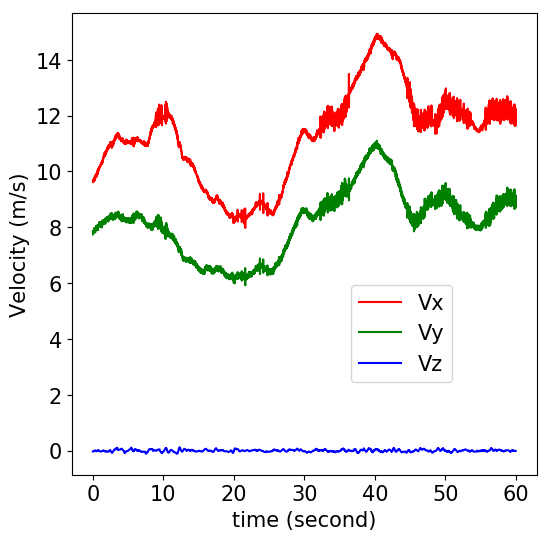}
        \caption{}
        \label{fig: kaist34 velocity}
    \end{subfigure}
    \hfill
    \begin{subfigure}{0.23\textwidth}
        \centering
          \includegraphics[width=\textwidth, height=0.8\textwidth]{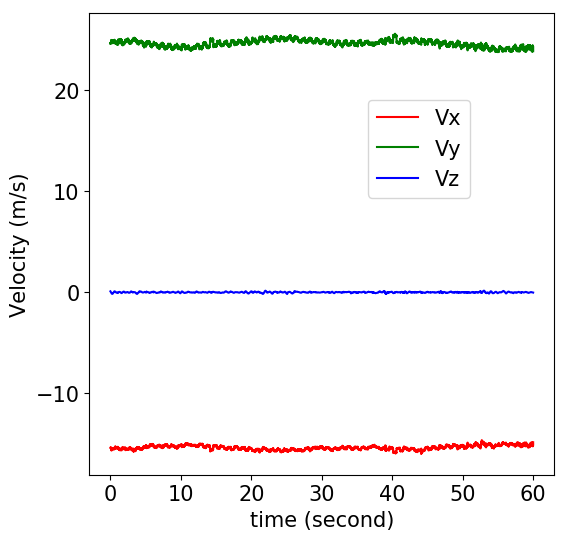}
          \caption{}
          \label{fig: kaist22 velocity}
    \end{subfigure}
  \caption{Velocity profiles of Urban34 (a) and Urban22 (b).}
  \label{fig: kaist velocity profile}
\end{figure}

Straight line motions are quite common in real-world scenarios. On one hand, straight line cruise is the most efficient and energy-saving trajectory for most robot applications. On the other hand, substantial artificial scenarios have specific constraints on motion, such as applications in agriculture, warehousing, logistics, and transportation.

The KAIST urban dataset \cite{jeong2019complex} contains the driving scenario on the highway, as shown in the \cref{fig: kaist picture}. Urban34 and Urban22 from this dataset are leveraged to confirm our observability finding, as these two sequences represent variable velocity motion and constant velocity motion, respectively. The vehicle used to collect data follows the same lane during driving, so its trajectory can be regarded as a pure translational straight line. Corresponding ${}_I^CRd$ is
$\begin{array}{l}
    \qquad\qquad\qquad {}_I^CRd = \left[ {\begin{array}{*{20}{c}}
    -0.00413\\
    -0.01966\\
    0.99980
    \end{array}} \right]
\end{array}$

The velocity curve of Urban34 sequence (see \cref{fig: kaist34 velocity}) is variable over time. \cref{fig: urban34 results} shows the calibration results with the pure VIO configuration and the global-pose aided VIO configuration.
Roll and pitch exhibit observable characteristic, while yaw not. This is because the non-zero component of ${}_I^CRd$ is dominated by the yaw component (0.99980).
The velocity curve of Urban22 sequence (see \cref{fig: kaist22 velocity}) is approximately constant. \cref{fig: urban22 results} shows the calibration results of Urban22.
In the pure VIO configuration, roll, pitch, and yaw all exhibit unobservable characteristic due to constant velocity motion. In the global-pose aided VIO configuration, unobservable degrees of freedom are reduced from 3 to 1 (yaw).
Interestingly, the convergence error of pitch is larger than that of roll, which can be attributed to the fact that the absolute value of pitch component (0.01966) is larger than that of roll (0.00413). 

Furthermore, we evaluate the localization accuracy with calibration (w. calib) and without calibration (wo. calib), under different perturbations on the rotational extrinsic parameter. Since real-world data is more sensitive than simulation data, the perturbation amplitude is reduced to half of its value listed in the \cref{tab: variable velocity-vio}. The Absolute Trajectory Error (ATE) results are reported in \cref{tab:ate_vio} and \cref{tab:ate_global_vio}. 

In the pure VIO configuration (\cref{tab:ate_vio}), calibration significantly improves the localization accuracy for Urban34, as model error from two degrees of freedom (roll and pitch) of the rotational calibration parameter can be corrected to near 0, thanks to online calibration (see top of \cref{fig: urban34 results}). Urban22 exhibits large localization error as scale becomes unobservable under constant velocity motion \cite{delaune2021range}.
And it is observed that performing calibration further degrades the localization due to the fully unobservable property of the rotational calibration parameter (see top of \cref{fig: urban22 results}). In the global-pose aided VIO configuration (\cref{tab:ate_global_vio}), the localization accuracy is mainly dominated by global pose measurements, thus the calibration of rotational extrinsic parameter has negligible impact on the accuracy.

\begin{remark}
If the calibration parameter is observable, online calibration typically brings positive benefits to localization \cite{song2024gps}. However, if it is unobservable, the impact of calibration on localization is unpredictable (negative, no impact or positive). In other words, we cannot determine the observability of the calibration parameter from localization accuracy.
\end{remark}

\begin{figure*}
  \centering
  \begin{subfigure}{0.96\linewidth}
    \includegraphics[width=\textwidth, height=0.2\textwidth]{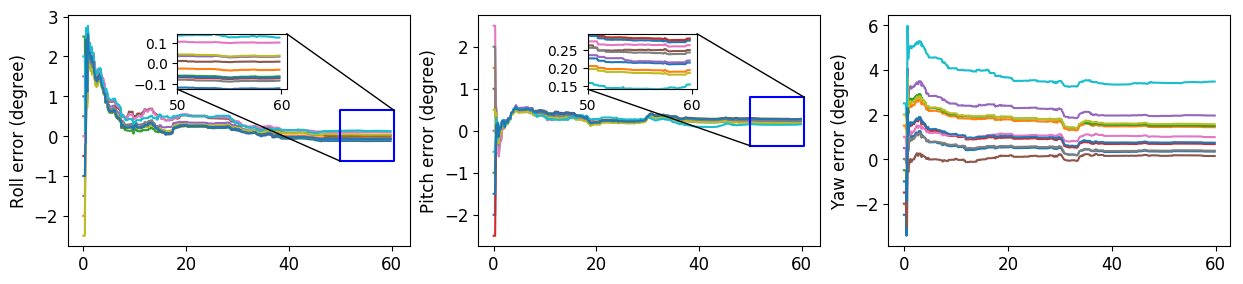}
  \end{subfigure}
  
  \begin{subfigure}{0.96\linewidth}
    \includegraphics[width=\textwidth, height=0.2\textwidth]{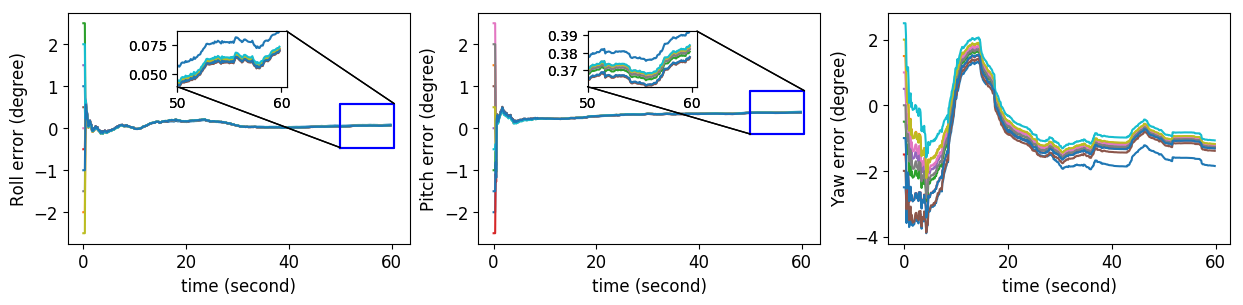}
  \end{subfigure}

  \caption{Calibration results for Urban34. Top: Results for pure VIO system. Bottom: Results for global-pose aided VIO system. $y$-axis represents errors of the rotational calibration parameter over time respect to different initial guesses. $x$-axis represents time in seconds.}
  \label{fig: urban34 results}
  
\end{figure*}

\begin{figure*}
  \centering
  \begin{subfigure}{0.96\linewidth}
    \includegraphics[width=\textwidth, height=0.2\textwidth]{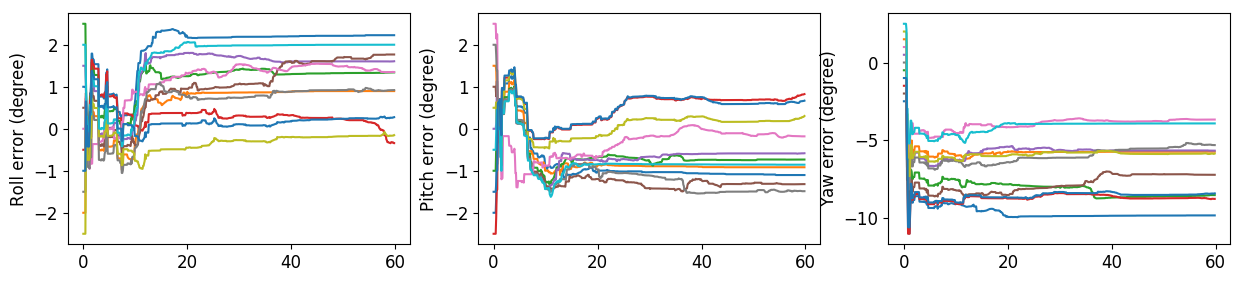}
  \end{subfigure}
  
  \begin{subfigure}{0.96\linewidth}
    \includegraphics[width=\textwidth, height=0.2\textwidth]{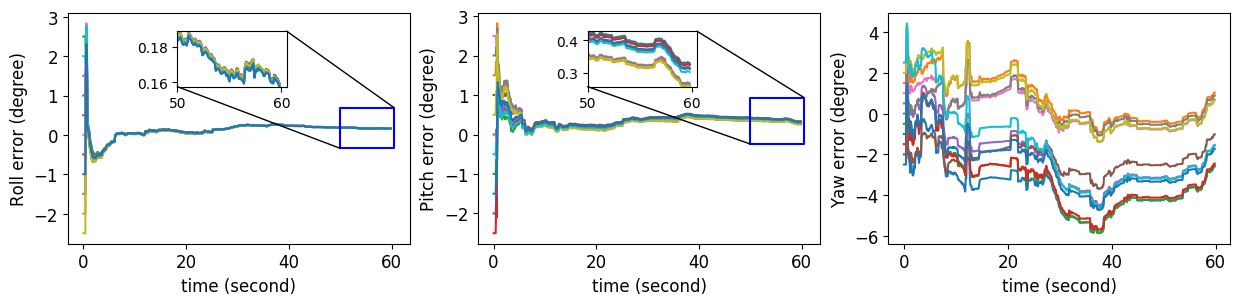}
  \end{subfigure}

  \caption{Calibration results for Urban22. Top: Results for pure VIO system. Bottom: Results for global-pose aided VIO system. $y$-axis represents errors of the rotational calibration parameter over time respect to different initial guesses. $x$-axis represents time in seconds.}
  \label{fig: urban22 results}
  
\end{figure*}

\begin{table}
  \caption{ATE (meter) Comparison for Pure VIO.}
  \centering
  \scalebox{1.0}{
  \begin{tabular}{@{}ccccc@{}}
    \toprule
    \multirow{2}{*}{Perturbations} & \multicolumn{2}{c}{Urban34} & \multicolumn{2}{c}{Urban22} \\
    \cmidrule(lr){2-3} \cmidrule(lr){4-5}
    & {w. calib} & {wo. calib} & {w. calib} & {wo. calib} \\
    \midrule
    (1.0, -2.0, -2.5) & 2.49 & 13.27 & 257.52 & 212.16 \\
    (-2.0, 1.5, 1.5) & 4.24 & 108.88 & 258.47 & 88.39 \\
    (2.5, -1.0, -0.5) & 2.90 & 106.31 & 268.02 & 308.66 \\
    (-0.5, -2.5, -1.5) & 2.38 & 53.36 & 125.40 & 90.07 \\
    (1.5, 0.0, 0.5) & 10.23 & 36.76 & 282.57 & 119.22 \\
    (0.5, 1.0, -2.0) & 9.78 & 5.31 & 238.12 & 102.42 \\
    (0.0, 2.5, 1.0) & 7.65 & 49.58 & 266.51 & 128.76 \\
    (-1.5, 2.0, 0.0) & 2.30 & 80.54 & 198.29 & 54.08 \\
    (-2.5, 0.5, 2.0) & 5.54 & 112.60 & 78.57 & 97.87 \\
    (2.0, -0.5, 2.5) & 39.81 & 127.98 & 248.98 & 335.19 \\
    (-1.0, -1.5, -1.0) & 2.37 & 45.99 & 94.43 & 56.92 \\
    \hline
    Avg & 8.15 & 67.33 & 210.63 & 144.89\\
    \bottomrule
  \end{tabular}}
  \label{tab:ate_vio}
\end{table}


\begin{table}
  \caption{ATE (meter) Comparison for Global-pose aided VIO.}
  \centering
  \scalebox{1.0}{
  \begin{tabular}{@{}ccccc@{}}
    \toprule
    \multirow{2}{*}{Perturbations} & \multicolumn{2}{c}{Urban34} & \multicolumn{2}{c}{Urban22} \\
    \cmidrule(lr){2-3} \cmidrule(lr){4-5}
    & {w. calib} & {wo. calib} & {w. calib} & {wo. calib} \\
    \midrule
    Avg & 0.41 & 0.42 & 0.19 & 0.19\\
    \bottomrule
  \end{tabular}}
  \label{tab:ate_global_vio}
\end{table}

\section{Conclusion}

We investigate the observability from \cite{yang2019degenerate, yang2023online}, and prove that the common-seen pure translational straight line motion can lead to the unobservability of the rotational extrinsic parameter between IMU and camera (at least one degree of freedom).
Our novel finding is carefully verified through rigorous theory, numerical study, and real-world experiment. 
This finding makes up for the shortcomings of the existing research conclusions.
When the observability conclusion is inconsistent with the numerical study results (see our comments in \cref{sec: Comments}), we recommend:

\begin{itemize}

\item Perform ablation experiments to eliminate the influence of other calibration parameters.

\item Try different initial values to test the convergence consistency of the interested calibration parameter.

\end{itemize}

Mathematical derivations of this paper and \cite{yang2019degenerate} require delicate search for the null space of the observability matrix. And this process is case by case, which prompts us a research question for future work. Is there an automatic and natural way to find degenerate motion and corresponding unobservable degrees of freedom, thus avoiding potential manual missing or mistake?









\bibliographystyle{ieeetr}
\bibliography{bib}

\clearpage
\setcounter{page}{1}
\maketitlesupplementary

\section{Correction of the observability matrix for global-pose aided VIO} \label{sec: additional_correction}

The observability matrix plays a key role for the observability analysis of a linear or nonlinear state estimator. According to the Section II.E of \cite{yang2019degenerate}, the measurement Jacobian matrix and state transition matrix need to be calculated in advance to construct the observability matrix. For ease of description, recall our state vector (\cref{eq:state})
\begin{equation}
    x = {\left[ {\begin{array}{*{20}{c}}
    {{}_G^I{q^T}}&{b_g^T}&{{}^Gv_I^T}&{b_a^T}&{{}^Gp_I^T}&{{}_I^C{q^T}}&{{}^Gp_f^T}
    \end{array}} \right]^T}
\end{equation}

According to the Section II.D of \cite{yang2019degenerate}, the measurement Jacobian matrix corresponding to global pose measurement can be calculated as
\begin{equation}
    {H_{{V_k}}} = \left[ {\begin{array}{*{20}{c}}
    {{I_3}}&{{0_3}}&{{0_3}}&{{0_3}}&{{0_3}}&{{0_3}}&{{0_3}}\\
    {{0_3}}&{{0_3}}&{{0_3}}&{{0_3}}&{{I_3}}&{{0_3}}&{{0_3}}
    \end{array}} \right]
\end{equation}

Combining the measurement Jacobian matrix corresponding to visual measurement, ${{H_{{C_k}}}}$, the overall measurement Jacobian matrix can be denoted as
\begin{equation}
    {H_k} = \left[ {\begin{array}{*{20}{c}}
    {{H_{{C_k}}}}\\
    {{H_{{V_k}}}}
    \end{array}} \right]
\end{equation}

Referring to equation (5) of \cite{yang2019degenerate}, the expression of our state transition matrix is
\begin{equation}
    \Phi \left( {k,1} \right) = \left[ {\begin{array}{*{20}{c}}
    {{\Phi _{I11}}}&{{\Phi _{I12}}}&{{0_3}}&{{0_3}}&{{0_3}}&{{0_3}}&{{0_3}}\\
    {{0_3}}&{{I_3}}&{{0_3}}&{{0_3}}&{{0_3}}&{{0_3}}&{{0_3}}\\
    {{\Phi _{I31}}}&{{\Phi _{I32}}}&{{I_3}}&{{\Phi _{I34}}}&{{0_3}}&{{0_3}}&{{0_3}}\\
    {{0_3}}&{{0_3}}&{{0_3}}&{{I_3}}&{{0_3}}&{{0_3}}&{{0_3}}\\
    {{\Phi _{I51}}}&{{\Phi _{I52}}}&{{\Phi _{I53}}}&{{\Phi _{I54}}}&{{I_3}}&{{0_3}}&{{0_3}}\\
    {{0_3}}&{{0_3}}&{{0_3}}&{{0_3}}&{{0_3}}&{{I_3}}&{{0_3}}\\
    {{0_3}}&{{0_3}}&{{0_3}}&{{0_3}}&{{0_3}}&{{0_3}}&{{I_3}}
    \end{array}} \right]
\end{equation}

Finally, the observability matrix for global-pose aided VIO can be constructed by multiplying the measurement Jacobian matrix with the state transition matrix
\begin{equation}
    \begin{array}{l}
    M_k^{\left( g \right)} = {H_k}\Phi \left( {k,1} \right)\\
     = \left[ {\begin{array}{*{20}{c}}
    {{H_{{C_k}}}\Phi \left( {k,1} \right)}\\
    {{H_{{V_k}}}\Phi \left( {k,1} \right)}
    \end{array}} \right]\\[10pt]
     = \left[ {\begin{array}{*{20}{c}}
    {{\Xi _k}{\Xi _{{\Gamma _k}}}}\\
    {{H_{{V_k}}}\Phi \left( {k,1} \right)}
    \end{array}} \right]\\[10pt]
     = \left[ {\begin{array}{*{20}{c}}
    {{\Xi _k}}&0\\
    0&{{I_6}}
    \end{array}} \right]\\
    \quad \times \left[ {\begin{array}{*{20}{c}}
    {{\Gamma _1}}&{{\Gamma _2}}&{ - {I_3}\delta {t_k}}&{{\Gamma _3}}&{ - {I_3}}&{{\Gamma _4}}&{{I_3}}\\
    {{\Phi _{I11}}}&{{\Phi _{I12}}}&{{0_3}}&{{0_3}}&{{0_3}}&{{0_3}}&{{0_3}}\\
    {{\Phi _{I51}}}&{{\Phi _{I52}}}&{{\Phi _{I53}}}&{{\Phi _{I54}}}&{{I_3}}&{{0_3}}&{{0_3}}
    \end{array}} \right]
    \end{array}
\end{equation}

This completes the correction for the equation (40) of \cite{yang2019degenerate}. 

\section{Additional results on numerical study} \label{sec: additional_numerical_results}

In this section, we will analyze additional calibration results described in \cref{table_Numerical} to complete the validation of our observable conclusions.

\cref{tab: variable velocity-global vio} shows the final calibration results of Trajectory-1 in the global-pose aided VIO configuration. For Case-1, pitch and yaw exhibit observable characteristic, while roll not. This is because the non-zero component of ${}_I^CRd$ corresponds to roll. For Case-2, yaw exhibits observable characteristic, while roll and pitch not. This is because non-zero components of ${}_I^CRd$ correspond to roll and pitch. For Case-3, roll, pitch, and yaw all exhibit unobservable characteristic. This is because none of the three components of ${}_I^CRd$ are zero. The calibration results over time are shown in the \cref{fig: variable velocity-global vio}.

\cref{tab: constant velocity-vio} shows the final calibration results of Trajectory-2 in the pure VIO configuration. For Case-1, Case-2, 
 and Case-3, roll, pitch, and yaw all exhibit unobservable characteristic. This can be explained by \cref{lemma2}, which indicates constant velocity motion lead to the fully unobservable property of the rotational extrinsic parameter. The calibration results over time are shown in the \cref{fig: constant velocity-vio}.

 \cref{tab: constant velocity-global vio} shows the final calibration results of Trajectory-2 in the global-pose aided VIO configuration. We can still observe that the convergence of the rotational extrinsic parameter, depends on which components of ${}_I^CRd$ are 0. The calibration results over time are shown in the \cref{fig: constant velocity-global vio}.

 These calibration results, and the corresponding observability conclusion they supported, are summarized in \cref{table_Numerical}. Extensive experimental results demonstrate the correctness of our novel theoretical finding (see \cref{table_Correction}).

\begin{table*}
  \caption{Final calibration results of the rotational extrinsic parameter for global-pose aided VIO system undergoes pure translational straight line motion with variable velocity. The absolute errors of roll, pitch, and yaw at 60s, are recorded with different perturbations.}
  \centering
  \begin{tabular}{@{}cccccccccc@{}}
    \toprule
    \multirow{2}{*}{\makecell{Perturbations of\\ (roll, pitch, yaw)}} & \multicolumn{3}{c}{Case-1} & \multicolumn{3}{c}{Case-2} & \multicolumn{3}{c}{Case-3} \\
    \cmidrule(lr){2-4} \cmidrule(lr){5-7} \cmidrule(lr){8-10}
    & {Roll} & {Pitch} & {Yaw} & {Roll} & {Pitch} & {Yaw} & {Roll} & {Pitch} & {Yaw} \\
    \midrule
    (2, -4, -5) & 0.89 & 0.07 & 0.01 & 4.00 & 1.73 & 0.02 & 4.98 & 1.84 & 2.15\\
    (-4, 3, 3) & 12.49 & 0.08 & 0.01 & 9.44 & 3.98 & 0.12 & 0.89 & 0.28 & 0.37\\
    (5, -2, -1) & 1.22 & 0.08 & 0.01 & 1.05 & 0.50 & 0.01 & 7.34 & 2.71 & 3.21\\
    (-1, -5, -3) & 5.82 & 0.08 & 0.01 & 6.82 & 2.89 & 0.06 & 3.62 & 1.33 & 1.55\\
    (3, 0, 1) & 4.31 & 0.08 & 0.01 & 2.87 & 1.26 & 0.02 & 6.01 & 2.22 & 2.62\\
    (1, 2, -4) & 5.33 & 0.08 & 0.01 & 4.83 & 2.07 & 0.03 & 2.92 & 1.06 & 1.25\\
    (0, 5, 2) & 8.68 & 0.08 & 0.01 & 5.59 & 2.39 & 0.05 & 2.98 & 1.08 & 1.29\\
    (-3, 4, 0) & 10.75 & 0.08 & 0.01 & 8.53 & 3.60 & 0.10 & 1.28 & 0.43 & 0.53\\
    (-5, 1, 4) & 13.39 & 0.08 & 0.01 & 10.38 & 4.37 & 0.14 & 0.92 & 0.29 & 0.38\\
    (4, -1, 5) & 4.49 & 0.08 & 0.01 & 1.88 & 0.85 & 0.01 & 8.02 & 2.96 & 3.52\\
    (-2, -3, -2) & 7.50 & 0.08 & 0.01 & 7.80 & 3.30 & 0.08 & 3.23 & 1.18 & 1.38\\
    \hline
    Avg & 6.81 & 0.08 & 0.01 & 5.74 & 2.45 & 0.06 & 3.84 & 1.40 & 1.66\\
    \bottomrule
  \end{tabular}
  \label{tab: variable velocity-global vio}
\end{table*}

\begin{figure*}
  \centering
  \begin{subfigure}{0.96\linewidth}
    \includegraphics[width=\textwidth, height=0.2\textwidth]{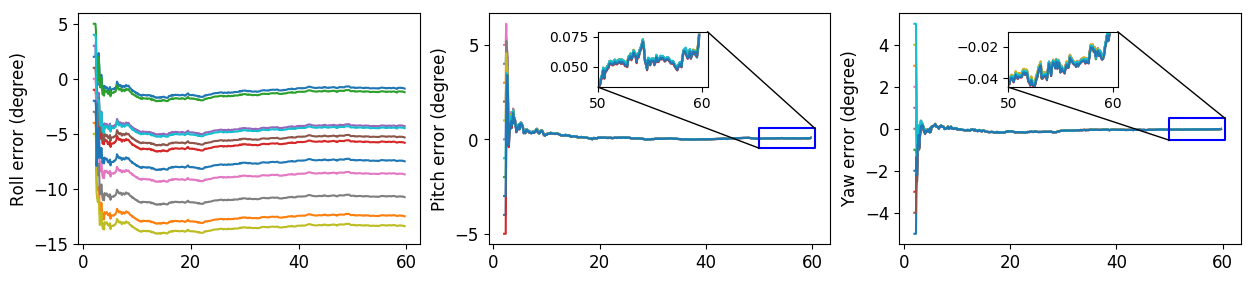}
  \end{subfigure}
  
  \begin{subfigure}{0.96\linewidth}
    \includegraphics[width=\textwidth, height=0.2\textwidth]{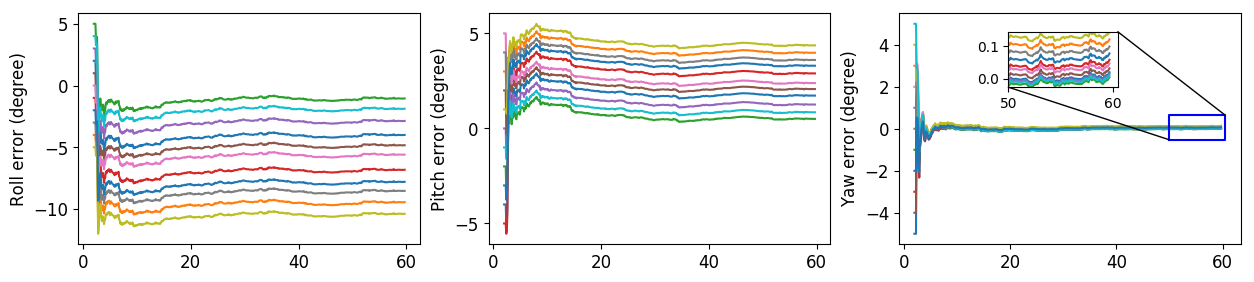}
  \end{subfigure}

  \begin{subfigure}{0.96\linewidth}
    \includegraphics[width=\textwidth, height=0.2\textwidth]{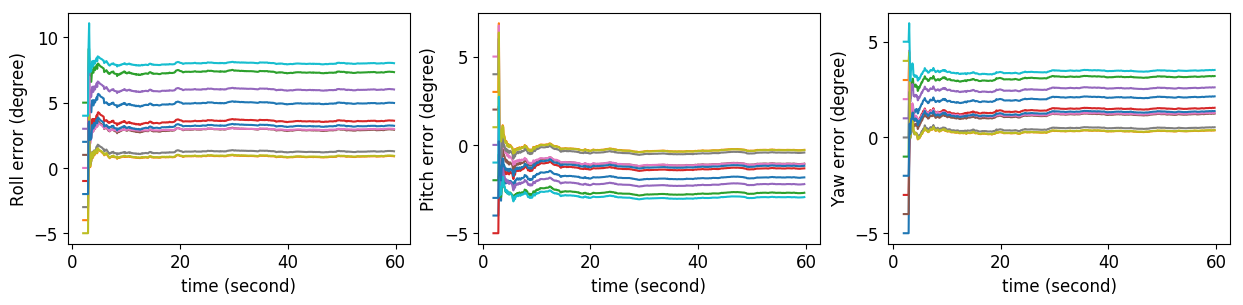}
  \end{subfigure}

  \caption{Calibration results for global-pose aided VIO system undergoes pure translational straight line motion with variable velocity. $y$-axis represents errors of the rotational calibration parameter over time respect to different initial guesses. $x$-axis represents time in seconds. Top to bottom corresponds to Case-1 to Case-3 in \cref{sec: Numerical Study}.}
  \label{fig: variable velocity-global vio}
  
\end{figure*}

\newpage
\begin{table*}
  \caption{Final calibration results of the rotational extrinsic parameter for pure VIO system undergoes pure translational straight line motion with constant velocity. The absolute errors of roll, pitch, and yaw at 60s, are recorded with different perturbations.}
  \centering
  \begin{tabular}{@{}cccccccccc@{}}
    \toprule
    \multirow{2}{*}{\makecell{Perturbations of\\ (roll, pitch, yaw)}} & \multicolumn{3}{c}{Case-1} & \multicolumn{3}{c}{Case-2} & \multicolumn{3}{c}{Case-3} \\
    \cmidrule(lr){2-4} \cmidrule(lr){5-7} \cmidrule(lr){8-10}
    & {Roll} & {Pitch} & {Yaw} & {Roll} & {Pitch} & {Yaw} & {Roll} & {Pitch} & {Yaw} \\
    \midrule
    (2, -4, -5) & 0.46 & 5.65 & 4.81 & 4.63 & 3.70 & 2.06 & 1.29 & 10.38 & 4.34\\
    (-4, 3, 3) & 2.77 & 1.82 & 3.77 & 7.74 & 2.02 & 2.03 & 5.97 & 2.39 & 6.96\\
    (5, -2, -1) & 5.40 & 6.66 & 1.41 & 7.18 & 3.41 & 0.20 & 21.48 & 10.81 & 1.41\\
    (-1, -5, -3) & 2.40 & 6.27 & 2.12 & 0.61 & 2.73 & 0.27 & 0.72 & 2.23 & 0.53\\
    (3, 0, 1) & 5.02 & 6.24 & 0.77 & 3.21 & 0.99 & 0.94 & 7.66 & 16.23 & 4.27\\
    (1, 2, -4) & 0.13 & 4.23 & 4.48 & 3.18 & 0.37 & 2.77 & 6.71 & 10.27 & 3.55\\
    (0, 5, 2) & 0.86 & 0.99 & 2.97 & 2.20 & 1.87 & 1.18 & 2.39 & 11.00 & 6.96\\
    (-3, 4, 0) & 1.63 & 3.45 & 0.41 & 3.57 & 2.04 & 0.25 & 2.33 & 3.93 & 1.67\\
    (-5, 1, 4) & 3.36 & 2.93 & 4.19 & 8.86 & 2.06 & 3.06 & 5.93 & 0.16 & 9.01\\
    (4, -1, 5) & 5.83 & 3.71 & 4.64 & 1.12 & 1.61 & 4.09 & 10.51 & 8.93 & 11.47\\
    (-2, -3, -2) & 2.57 & 6.09 & 1.25 & 0.77 & 1.56 & 0.17 & 2.92 & 2.72 & 1.42\\
    \hline
    Avg & 2.77 & 4.37 & 2.80 & 3.91 & 2.03 & 1.55 & 6.17 & 7.19 & 4.69\\
    \bottomrule
  \end{tabular}
  \label{tab: constant velocity-vio}
\end{table*}

\begin{figure*}
  \centering
  \begin{subfigure}{0.96\linewidth}
    \includegraphics[width=\textwidth, height=0.2\textwidth]{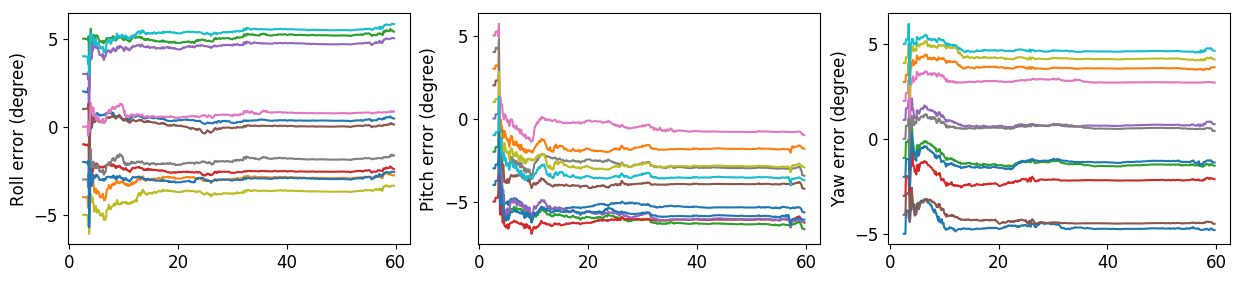}
  \end{subfigure}
  
  \begin{subfigure}{0.96\linewidth}
    \includegraphics[width=\textwidth, height=0.2\textwidth]{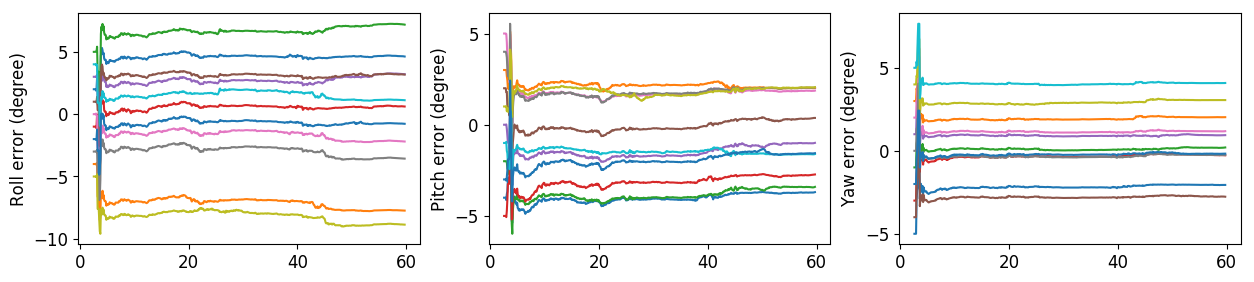}
  \end{subfigure}

  \begin{subfigure}{0.96\linewidth}
    \includegraphics[width=\textwidth, height=0.2\textwidth]{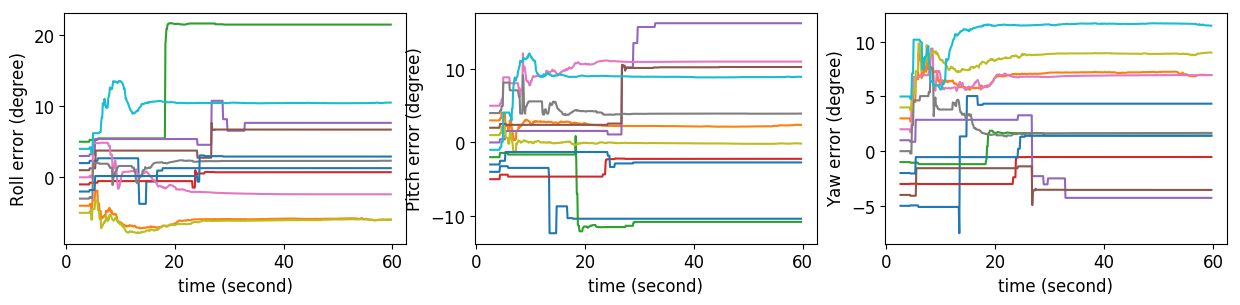}
  \end{subfigure}

  \caption{Calibration results for pure VIO system undergoes pure translational straight line motion with constant velocity. $y$-axis represents errors of the rotational calibration parameter over time respect to different initial guesses. $x$-axis represents time in seconds. Top to bottom corresponds to Case-1 to Case-3 in \cref{sec: Numerical Study}.}
  \label{fig: constant velocity-vio}
  
\end{figure*}

\begin{table*}
  \caption{Final calibration results of the rotational extrinsic parameter for global-pose aided VIO system undergoes pure translational straight line motion with constant velocity. The absolute errors of roll, pitch, and yaw at 60s, are recorded with different perturbations.}
  \centering
  \begin{tabular}{@{}cccccccccc@{}}
    \toprule
    \multirow{2}{*}{\makecell{Perturbations of\\ (roll, pitch, yaw)}} & \multicolumn{3}{c}{Case-1} & \multicolumn{3}{c}{Case-2} & \multicolumn{3}{c}{Case-3} \\
    \cmidrule(lr){2-4} \cmidrule(lr){5-7} \cmidrule(lr){8-10}
    & {Roll} & {Pitch} & {Yaw} & {Roll} & {Pitch} & {Yaw} & {Roll} & {Pitch} & {Yaw} \\
    \midrule
    (2, -4, -5) & 1.90 & 0.01 & 0.01 & 4.38 & 1.83 & 0.01 & 0.30 & 0.13 & 0.11\\
    (-4, 3, 3) & 4.56 & 0.00 & 0.02 & 2.39 & 1.00 & 0.00 & 8.82 & 3.22 & 3.90\\
    (5, -2, -1) & 1.88 & 0.00 & 0.01 & 6.66 & 2.76 & 0.03 & 9.69 & 3.57 & 4.22\\
    (-1, -5, -3) & 4.39 & 0.01 & 0.00 & 2.25 & 0.95 & 0.02 & 4.42 & 1.67 & 1.84\\
    (3, 0, 1) & 0.70 & 0.00 & 0.01 & 4.33 & 1.79 & 0.01 & 9.39 & 3.48 & 4.11\\
    (1, 2, -4) & 1.07 & 0.00 & 0.00 & 2.45 & 1.02 & 0.03 & 4.05 & 1.52 & 1.76\\
    (0, 5, 2) & 0.09 & 0.01 & 0.01 & 0.56 & 0.22 & 0.01 & 6.75 & 2.53 & 2.95\\
    (-3, 4, 0) & 3.47 & 0.01 & 0.01 & 1.46 & 0.61 & 0.01 & 4.06 & 1.52 & 1.76\\
    (-5, 1, 4) & 6.06 & 0.00 & 0.02 & 2.67 & 1.12 & 0.01 & 9.84 & 3.58 & 4.33\\
    (4, -1, 5) & 1.60 & 0.00 & 0.02 & 4.82 & 1.98 & 0.03 & 14.31 & 5.16 & 6.35\\
    (-2, -3, -2) & 4.78 & 0.01 & 0.01 & 0.96 & 0.41 & 0.02 & 2.50 & 0.93 & 1.04\\
    \hline
    Avg & 2.77 & 0.01 & 0.01 & 2.99 & 1.24 & 0.02 & 6.74 & 2.48 & 2.94\\
    \bottomrule
  \end{tabular}
  \label{tab: constant velocity-global vio}
\end{table*}

\begin{figure*}
  \centering
  \begin{subfigure}{0.96\linewidth}
    \includegraphics[width=\textwidth, height=0.2\textwidth]{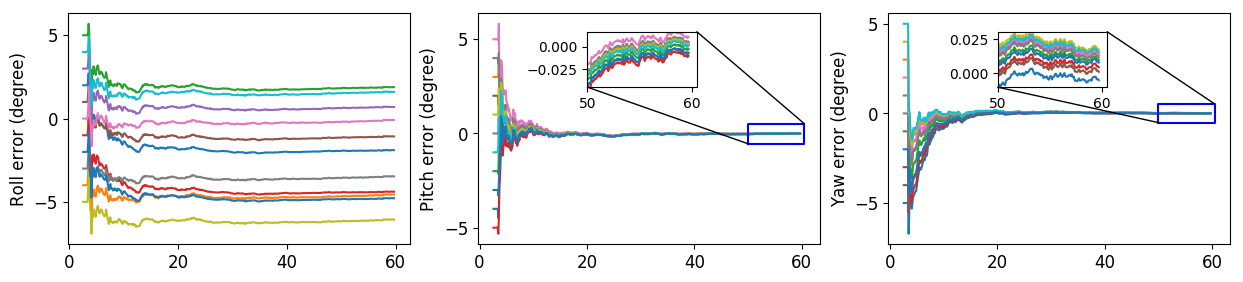}
  \end{subfigure}
  
  \begin{subfigure}{0.96\linewidth}
    \includegraphics[width=\textwidth, height=0.2\textwidth]{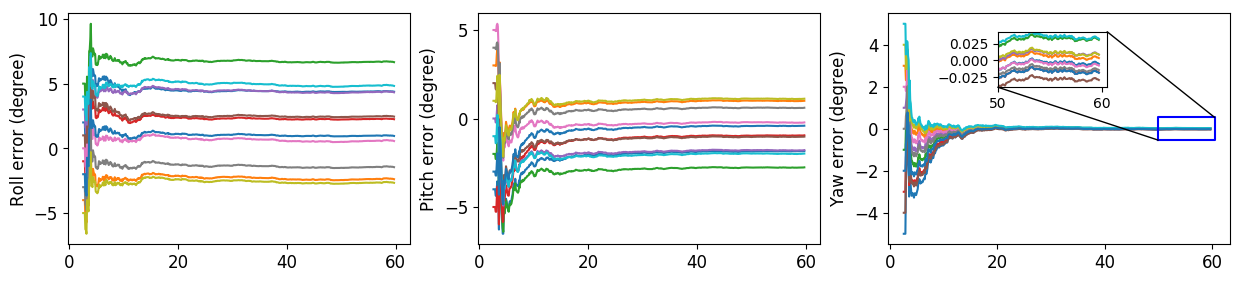}
  \end{subfigure}

  \begin{subfigure}{0.96\linewidth}
    \includegraphics[width=\textwidth, height=0.2\textwidth]{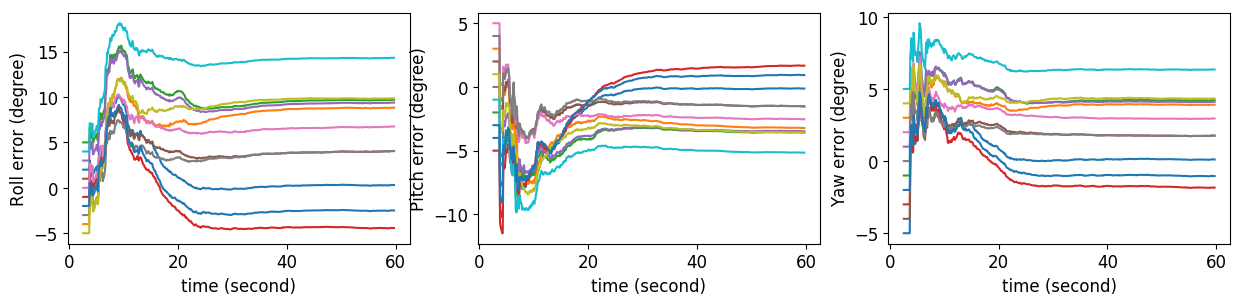}
  \end{subfigure}

  \caption{Calibration results for global-pose aided VIO system undergoes pure translational straight line motion with constant velocity. $y$-axis represents errors of the rotational calibration parameter over time respect to different initial guesses. $x$-axis represents time in seconds. Top to bottom corresponds to Case-1 to Case-3 in \cref{sec: Numerical Study}.}
  \label{fig: constant velocity-global vio}
  
\end{figure*}

\end{document}